%% file: main.tex
\documentclass{article}

\usepackage{pifont}

\usepackage{microtype}
\usepackage{graphicx}
\usepackage{subfigure}
\usepackage{booktabs} 

\usepackage{algorithm}
\usepackage{algpseudocode} 

\usepackage[accepted]{icml2024}

\input{math_supp}

\usepackage{amsmath}
\usepackage{amssymb}
\usepackage{mathtools}
\usepackage{amsthm}

\usepackage{xcolor}
\usepackage{thmtools,thm-restate}

\usepackage{etoc}
\etocframedstyle[1]{\textbf{\textsc{Table of Contents}}}
\etocsettocdepth{3}

\usepackage[textsize=tiny]{todonotes}

\icmltitlerunning{Modeling Choice via Self-Attention}

\begin{document}

\twocolumn[
\icmltitle{Modeling Choice via Self-Attention}



\icmlsetsymbol{equal}{*}

\begin{icmlauthorlist}
\icmlauthor{Joohwan Ko}{kaist}
\icmlauthor{Andrew A. Li}{cmu}
\end{icmlauthorlist}

\icmlaffiliation{kaist}{KAIST, Daejeon, Korea}
\icmlaffiliation{cmu}{Carnegie Mellon University, Pittsburgh, Pennsylvania, United States}

\icmlcorrespondingauthor{Joohwan Ko}{juanko@kaist.ac.kr}
\icmlcorrespondingauthor{Andrew A. Li}{aali1@cmu.edu}

\icmlkeywords{Machine Learning, ICML}

\vskip 0.3in
    ]



\printAffiliationsAndNotice{}  

\begin{abstract}
Models of choice are a fundamental input to many now-canonical optimization problems in the field of Operations Management, including assortment, inventory, and price optimization. Naturally, accurate estimation of these models from data is a critical step in the application of these optimization problems in practice. Concurrently, recent advancements in deep learning have sparked interest in integrating these techniques into choice modeling. However, there is a noticeable research gap at the intersection of deep learning and choice modeling, particularly with both theoretical and empirical foundations. Thus motivated, we first propose a choice model that is the first to successfully (both theoretically and practically) leverage a modern neural network architectural concept (self-attention). 
Theoretically, we show that our attention-based choice model is a low-rank generalization of the Halo Multinomial Logit (Halo-MNL) model.
We prove that whereas the Halo-MNL requires $\Omega(m^2)$ data samples to estimate, where $m$ is the number of products, our model supports a natural nonconvex estimator (in particular, that which a standard neural network implementation would apply) which admits a near-optimal stationary point with $O(m)$ samples. Additionally, we establish the first realistic-scale benchmark for choice model estimation on real data, conducting the most extensive evaluation of existing models to date, thereby highlighting our model's superior performance.

\end{abstract}

\input{sections/section_introduction}
\input{sections/section_related_works}
\input{sections/section_choice_model_as_NN}
\input{sections/section_our_model_2_lenses}
\input{sections/section_estimation_theory}
\input{sections/section_experiments}

\input{sections/section_conclusion}

\section*{Broader Impact}
This research contributes to the theoretical and empirical landscape of choice modeling, with potential societal implications in this field. However, we have not identified any significant ethical considerations or specific outcomes that require particular emphasis in our study.

\bibliography{icml24}
\bibliographystyle{icml2024}

\clearpage
\appendix
\onecolumn
\newpage
\clearpage
\section{Proofs}

\subsection{Proof of  \cref{prop:LB-HaloMNL}}

\printProofs[halomnlsamplecomplexity]

\subsection{Proof of \cref{prop:LB-LowRankHalo}}
\printProofs[lowrankhalomnlsamplecomplexity]


\subsection{Proof of \cref{thm:UB}}
\printProofs[lowrankestimation]

\input{sections/section_appendix_experiments}

\end{document}

%% file: math_supp.tex
\usepackage{amssymb}

\usepackage{amsmath,amsthm}
\usepackage{mathtools}
\usepackage{iftex}
\usepackage{etoolbox}

\ifxetex
  \usepackage{microtype}
  \usepackage{unicode-math}
  \newcommand{\mathmat}[1]{\mathbfit{#1}}
  \newcommand{\mathvec}[1]{\mathbfit{#1}}
  \newcommand{\mathvecgreek}[1]{\mathbfit{#1}}
  \newcommand{\mathmatgreek}[1]{\mathbfit{#1}}

  \newcommand{\mathrv}[1]{\mathsfit{#1}}
  \newcommand{\mathrvvec}[1]{\mathbfsfit{#1}}
  \newcommand{\mathrvgreek}[1]{\mathsfit{#1}}
  \newcommand{\mathrvvecgreek}[1]{\mathbfsfit{#1}}

  \usepackage{fontspec}
\else
  \usepackage[notext,lcgreekalpha]{stix2}
  \usepackage{newtxtext}
  \usepackage{textcomp}

  \newcommand{\mathmat}[1]{\mathbfit{#1}}
  \newcommand{\mathvec}[1]{\mathbfit{#1}}
  \newcommand{\mathvecgreek}[1]{\mathbfit{#1}}
  \newcommand{\mathmatgreek}[1]{\mathbfit{#1}}

  \newcommand{\mathrv}[1]{\mathsfit{#1}}
  \newcommand{\mathrvvec}[1]{\mathbfsfit{#1}}
  \newcommand{\mathrvgreek}[1]{\mathsfit{#1}}
  \newcommand{\mathrvvecgreek}[1]{\mathbfsfit{#1}}
\fi

\usepackage{booktabs,threeparttable,arydshln}
\usepackage{multirow}
\usepackage{nicematrix}

\usepackage{pgfplots}
\pgfkeys{/pgfplots/tuftelike/.style={
  semithick,
  tick style={major tick length=4pt,semithick,black},
  separate axis lines,
  axis x line*=bottom,
  axis x line shift=2pt,
  xlabel shift=-2pt,
  axis y line*=left,
  axis y line shift=2pt,
  ylabel shift=-2pt}}
  
\usepackage{proof-at-the-end}

\usepackage{mdframed}
\mdfdefinestyle {mdleftbarcoloredstyle} {%
    leftline          = true,%
    rightline         = false,%
    topline           = false,%
    bottomline        = false,%
    linewidth         = 2.5pt,%
    backgroundcolor   = gray!10,%
    linecolor         = gray,%
    innertopmargin    = 0.7ex,%
    innerbottommargin = 0.7ex,%
    innerleftmargin   = 1.ex,%
    ntheorem          = false,%
}

\mdfdefinestyle {coloredstyle} {%
    leftline          = false,%
    rightline         = false,%
    topline           = false,%
    bottomline        = false,%
    backgroundcolor   = white,%
    skipabove         = \parskip,%
    skipbelow         = 0,%
    innertopmargin    = 1.0ex,%
    innerbottommargin = 0.3ex,%
    innerleftmargin   = 1.ex,%
}

\declaretheoremstyle[
    spaceabove    = \parsep,
    spacebelow    = \parsep,
    bodyfont      = \normalfont\itshape,
]{theoremsty}

\declaretheorem[name=Proposition,style=theoremsty, mdframed={style = coloredstyle}]{proposition}

\declaretheorem[name=Theorem,    style=theoremsty, mdframed={style = coloredstyle}, numbered=no]{theorem*}
\declaretheorem[name=Proposition,style=theoremsty, mdframed={style = coloredstyle}, numbered=no]{proposition*}
\declaretheorem[name=Corollary,  style=theoremsty, mdframed={style = coloredstyle}, numbered=no]{corollary*}
\declaretheorem[name=Lemma,      style=theoremsty, mdframed={style = coloredstyle}, numbered=no]{lemma*}

\declaretheorem[name=Open Problem,style=theoremsty, mdframed={style = coloredstyle}, numbered=no]{openproblem*}
\declaretheorem[name=Conjecture,  style=theoremsty, mdframed={style = coloredstyle}, numbered=no]{conjecture*}

\declaretheoremstyle[
    spaceabove=\parsep,
    spacebelow=\parsep,
    bodyfont=\normalfont,
]{normalsty}

\declaretheorem[name=Remark,      style=normalsty, numbered=no]{remark*}
\declaretheorem[name=Definition,  style=normalsty, mdframed={style = coloredstyle}, numbered=no]{definition*}
\declaretheorem[name=Assumption,  style=normalsty, mdframed={style = coloredstyle}, numbered=no]{assumption*}

\newenvironment{proofsketch}{%
  \proof%
  }{\endproof}

\makeatletter
\providecommand\IfFormatAtLeastTF{\@ifl@t@r\fmtversion}
\IfFormatAtLeastTF{2020/10/02}{%
  \usepackage[bookmarksopen=true,bookmarks=true,colorlinks=true,backref=page]{hyperref}
  \usepackage[noabbrev, capitalise, nameinlink]{cleveref}
  \renewcommand*{\backref}[1]{}
  \renewcommand*{\backrefalt}[4]{({\footnotesize%
  \ifcase #1 Not cited.%
    \or page~#2%
    \else pages #2%
  \fi%
  })}

}{%
  \usepackage{hyperref}
  \usepackage[noabbrev, capitalise, nameinlink]{cleveref}%
}
\makeatother

\definecolor{linkcolor}{HTML}{6929C4}
\definecolor{citecolor}{HTML}{0043CE}
\hypersetup{colorlinks={true},linkcolor=linkcolor,citecolor=citecolor}

\crefname{assumption}{Assumption}{Assumption}
\crefname{condition}{Condition}{Condition}

\crefname{framedtheorem}{Theorem}{Theorems}
\crefname{framedproposition}{Proposition}{Propositions}
\crefname{framedlemma}{Lemma}{Lemmas}

\makeatletter
\def\adl@drawiv#1#2#3{%
        \hskip.5\tabcolsep
        \xleaders#3{#2.5\@tempdimb #1{1}#2.5\@tempdimb}%
                #2\z@ plus1fil minus1fil\relax
        \hskip.5\tabcolsep}
\newcommand{\cdashlinelr}[1]{%
  \noalign{\vskip\aboverulesep
           \global\let\@dashdrawstore\adl@draw
           \global\let\adl@draw\adl@drawiv}
  \cdashline{#1}
  \noalign{\global\let\adl@draw\@dashdrawstore
           \vskip\belowrulesep}}
\makeatother

\pgfkeys{/prAtEnd/global custom defaults/.style={
    end,
    restate,
    text proof={Proof},
    text link={\textit{Proof.} See the \hyperref[proof:prAtEnd\pratendcountercurrent]{\textit{full proof}} in page~\pageref{proof:prAtEnd\pratendcountercurrent}.}
  }
}

\DeclareMathOperator*{\argmax}{arg\,max}

\newcommand*\xbar[1]{%
  \hbox{%
    \vbox{%
      \hrule height 0.6pt 
      \kern0.33ex
      \hbox{%
        \kern-0.1em
        \ensuremath{#1}%
        \kern-0.1em
      }%
    }%
  }%
}

\def\do#1{\csdef{v#1}{\mathvec{#1}}}
\docsvlist{a,b,c,d,e,f,g,h,i,j,k,l,m,n,o,p,q,r,s,t,u,v,w,x,y,z}
\docsvlist{A,B,C,D,E,F,G,H,I,J,K,L,M,N,O,P,Q,R,S,T,U,V,W,X,Y,Z}

\def\do#1{\csdef{v#1}{\mathvecgreek{\csname#1\endcsname}}}
\docsvlist{alpha,beta,gamma,delta,epsilon,zeta,eta,theta,iota,kappa,lambda,mu,nu,xi,omikron,pi,rho,sigma,tau,upsilon,phi,chi,psi,omega}

\def\do#1{\csdef{rv#1}{\mathrv{#1}}}
\docsvlist{a,b,c,d,e,f,g,h,i,j,k,l,m,n,o,p,q,r,s,t,u,v,w,x,y,z}
\docsvlist{A,B,C,D,E,F,G,H,I,J,K,L,M,N,O,P,Q,R,S,T,U,V,W,X,Y,Z}

\def\do#1{\csdef{rv#1}{\mathrvgreek{\csname#1\endcsname}}}
\docsvlist{alpha,beta,gamma,delta,epsilon,zeta,eta,theta,iota,kappa,lambda,mu,nu,xi,omikron,pi,rho,sigma,tau,upsilon,phi,chi,psi,omega}

\def\do#1{\csdef{rvv#1}{\mathrvvec{#1}}}
\docsvlist{a,b,c,d,e,f,g,h,i,j,k,l,m,n,o,p,q,r,s,t,u,v,w,x,y,z}

\def\do#1{\csdef{rvv#1}{\mathrvvecgreek{\csname#1\endcsname}}}
\docsvlist{alpha,beta,gamma,delta,epsilon,zeta,eta,theta,iota,kappa,lambda,mu,nu,xi,omikron,pi,rho,sigma,tau,upsilon,phi,chi,psi,omega}

\def\do#1{\csdef{rvm#1}{\mathrvvecgreek{#1}}}
\docsvlist{A,B,C,D,E,F,G,H,I,J,K,L,M,N,O,P,Q,R,S,T,U,V,W,X,Y,Z}

\def\do#1{\csdef{m#1}{\mathmat{#1}}}
\docsvlist{A,B,C,D,E,F,G,H,I,J,K,L,M,N,O,P,Q,R,S,T,U,V,W,X,Y,Z}

\def\do#1{\csdef{m#1}{\mathmatgreek{\csname#1\endcsname}}}
\docsvlist{Sigma,Lambda,Phi,Gamma,Theta,Delta}

%% file: sections/section_introduction.tex
\section{Introduction}

\input{sections/table_related_works}

Consider the following problem which brick-and-mortar retailers frequently face: for any product category (e.g.~coffee, shampoo, etc.), the retailer must select an assortment of products to offer. Naturally, the retailer seeks to maximize profit, but has to contend with operational constraints such as limited shelf space, along with the intricacies of human {\em choice} (e.g.~it may be sub-optimal to include a popular product if it is low-profit and cannibalizes demand for other high-profit products).

Fundamental to this operation is a mathematical model of customer choice, or {\em choice model}. At this point, we might expect the story to play out in a familiar manner: the retailer uses past data to {\em estimate} a choice model, and in particular, decades-worth of experience from deep learning in other applications is brought to bear. 

Simply put, this has {\em not} been the case. Specifically, two issues particularly stand out vis-\`a-vis other well-established application domains for deep learning. First, there is overwhelming emphasis placed on estimating choice models which are {\em interpretable}. Now for a choice model to be ``interpretable'' here  does not refer to the notion with the same name that is studied in ongoing research by the deep learning community, but instead requires the choice model to be justifiable via (often heuristic) evidence on ``how'' humans make decisions. We contend that this interpretability is satisfied at the cost of model {\em accuracy}, and in particular is the reason why, for example, neural networks and deep learning largely do {\em not} appear in  choice modeling.

The second issue is that, to date, there does not exist a standard benchmark for choice estimation in the way that most deep learning applications are evaluated. Indeed, arguably even a standard {\em dataset} does not exist -- the datasets previously used to evaluate choice estimation are either unrealistically small or not publicly available. 
The closest to an ideal dataset is the {\em IRI Academic Dataset} \cite{bronnenberg2008database}, though previous experiments have restricted analysis to {\em two weeks} of data (out of multiple years!), and {\em ten} products within each category (out of more than twenty). 

Succinctly, the purpose of this paper is to, in many ways for the first time, treat choice model estimation as a modern deep learning application. To that end, we accomplish the following:
\begin{enumerate}
    \item {\em A New Choice Model based on Self-Attention:} We propose a new choice model which simultaneously (a) generalizes an existing choice model to deal with its intractable sample complexity, and (b) is representable (and trainable) as a simple neural network with a single self-attention head. Drawing this connection yields a model which enjoys both interpretability and the potential for the high model accuracy expected of a modern neural network architectural element. 
    \item {\em Theoretical Guarantees:} We prove that while accurate estimation of the original choice model that we build on requires a volume of data that scales {\em quadratically} in the number of products, our choice model can be estimated with a {\em linear} volume of data. We do this via analysis of a natural regularized nonconvex estimator, the same that would be used in a neural network implementation.
    \item {\em A New Benchmark:} We establish a new benchmark for choice model estimation. Our benchmark\footnote{We will release our code publicly available.} uses the same IRI dataset previously used by others, but substantially expands the range of data used to include an entire year, all 31 product categories, and up to 20 products within each category. This increases, by at least an order of magnitude, the number of transactions and unique assortments that can be analyzed. 
    Using this data, we evaluate a nearly-comprehensive set of existing choice models, along with our proposed model, and find that ours empirically dominates the rest.
\end{enumerate}

%% file: sections/table_related_works.tex
\begin{table*}[t]
\caption{Comparison of theoretical guarantees, experimental evidence and irrationality in current and previous studies}
\label{table:related_works}
\vskip 0.15in
\begin{center}
\begin{small}
\begin{sc}
\begin{tabular}{lccccr} 
\toprule
Method & \multicolumn{1}{p{2cm}}{\centering Theoretical \\ Guarantee} & 
\multicolumn{1}{p{2cm}}{\centering Experimental \\ Evidence}  & Irrationality \\
\midrule
Mixed MNL \citep{mcfadden2000mixed}& {\color{green}\ding{51}}& {\color{red}\ding{55}} &
{\color{red}\ding{55}}
\\
Halo MNL \citep{maragheh2018customer}& {\color{red}\ding{55}}& {\color{red}\ding{55}} &
{\color{green}\ding{51}}
\\
Attention-Based$^1$ \citep{li2021choice}& {\color{red}\ding{55}}& {\color{red}\ding{55}}&
{\color{green}\ding{51}}
\\
Attention-Based$^2$ \citep{phan2022attentionchoice}& {\color{red}\ding{55}}& {\color{red}\ding{55}}&
{\color{green}\ding{51}}
\\
Ours & {\color{green}\ding{51}}& {\color{green}\ding{51}}&
{\color{green}\ding{51}}
\\
\bottomrule
\end{tabular}
\end{sc}
\end{small}
\end{center}
\vskip -0.1in
\end{table*}

%% file: sections/section_related_works.tex
\section{Related Work}

The literature on choice models is vast and spans a broad range of areas, including Machine Learning, Behavioral Economics, Operations Research, and Transportation, just to name a few. This body of work includes traditional models like the multinomial logit (MNL) model \citep{luce1959individual, mcfadden1973conditional}, the latent-class MNL \citep{mcfadden1973conditional}, and the nested logit model \citep{williams1977formation}. These models typically adhere to the \textit{random utility maximization} (RUM) principle. In contrast, recent studies have ventured beyond this traditional \textit{rational} framework. They aim to account for \textit{irrational} aspects of choice behavior, as seen in models like the Halo-MNL \citep{maragheh2018customer} and the generalized stochastic preference model \citep{berbeglia2018generalized}. Moreover, there exists a subset of literature that is superficially "most relevant" to our study, specifically those that address some combination of neural networks and choice models.

\citet{aouad2022representing} model an existing choice model (the mixture MNL) as a neural network; we will discuss this in detail in the next section. \citet{cai2022deep} propose new choice models based on natural neural network architectures: we will find in our own empirical evaluations that this approach is dominated by existing choice models. \citet{phan2022attentionchoice} apply a cross-attention mechanism, which is similar in spirit to our own model (which applies self-attention). \citet{li2021choice} apply a self-attention mechanism similar to ours. In contrast to all of these works, ours is the first to demonstrate (a) theoretical success via a provable estimation guarantee, and (b) empirical success on a realistic experimental benchmark. Finally, \citet{gallego2017attention} propose a choice model which captures ``attention,'' but ``attention'' there refers to the human notion, rather than the neural network architectural element.

\subsection{Relation to Halo-MNL} Our model is indeed a generalization of the existing Halo-MNL model. But note that before our paper, results for the Halo-MNL were quite limited: theoretically, there existed an estimation algorithm but {\em no} sample complexity guarantee (incidentally, our proposed estimator is {\em not} just a straight-forward generalization of this), and experimentally, the Halo-MNL was {\em not} found to be the most accurate choice model (incidentally, on datasets that are orders of magnitude smaller than ours). Thus, with respect to the Halo-MNL, our theoretical results are entirely novel, and the generalization we make is of substantial empirical value. 

\subsection{Relation to Attention-Based Approaches:} Our model draws a connection to self-attention, and indeed, there is {\em one} previous paper \citep{li2021choice} that attempted to apply self-attention to choice modeling. That paper, in fact, proposes self-attention as a generic mechanism for choice modeling, and so our model is a special case. But that paper contains {\em no} theoretical guarantee and {\em no} empirical comparison to existing choice models.  

 For a quick comparison at a glance, we have compared theoretical guarantee, experimental evidence, and irrationality for the most \textit{related} works in the \cref{table:related_works}. 

%% file: sections/section_choice_model_as_NN.tex
\section{Choice Models as Neural Networks}\label{chapter3}
Let $[m] = \{1,\ldots,m\}$ denote a set of ``products.'' A {\em choice function} is a mapping $p: \{0,1\}^m \to \Delta^{m-1}$ such that for any $a \in \{0,1\}^m$, we have that $p(a)_j = 0$ if $a_j = 0$.
That is, the binary vector $a$ encodes the products that make up an {\em assortment}, and $p(a)$ encodes the probability that each product is chosen when assortment $a$ is offered (exactly one product is chosen, and hence they sum to one). 

The problem of {\em choice estimation} consists of estimating a ``ground truth'' choice function $p^*$ from data in the form of {\em transactions}, or tuples $(a_i,y_i)$ where $a_1,\ldots,a_n \in \{0,1\}^m$ are assortments, and $y_1,\ldots,y_n \in \Delta^{m-1}$ are choices made independently according to $p^*$:
\[  y_i = e_j \;\; \text{with probability} \;\; p^*(a_i)_j.\]
Since the space of all choice functions is massive (e.g.~its dimension is at least $\Omega(2^m)$ when viewed as a vector space), much work has been devoted to positing smaller sub-families, called {\em choice models}, of this space that strikes a balance between simplicity of estimation and representability of realistic choice-making. Before describing our proposed choice model in the next section, we use this section to describe three existing choice models that are widely used and admit natural neural network representations.


1. {\bf MNL:} Perhaps the simplest, most-popular choice model is the {\em multinomial logit} ({\em MNL}), which is entirely parameterized by a single vector $\alpha \in \mathbb{R}^m$:
\[
p(a)_i = \frac{\exp(\alpha_i) }{ \sum_{j \in [m]} a_j\exp(\alpha_j)}, \;\; \forall i: a_i = 1
\]
i.e.~the choice probabilities of an assortment $a$ are computed as the $\mathrm{softmax}(\cdot)$ of the corresponding elements of $\alpha$. In the parlance of neural networks, MNL can be represented as a network with no hidden layers. 

2. {\bf Mixture MNL:} One useful generalization of the MNL is to {\em mixtures} whose components are taken from MNL. A classic result \cite{mcfadden2000mixed} in choice modeling states that this family is rich enough to encompass all ``rational'' choice functions.\footnote{``Rational'' loosely means that choices are made according to the maximization of some (random) utility.} As previously observed \cite{aouad2022representing}, a finite mixture of MNL components can again be represented as a neural network, this time with a single hidden layer. 

3. {\bf Halo MNL:}
Another generalization of MNL is the {\em Halo MNL} \citep{maragheh2018customer}, which is designed to model simple irrational behaviors called  ``halo'' effects. It is parameterized by a matrix $H \in \mathbb{R}^{m \times m}$:
\begin{equation} \label{eqn:HaloMNL}
p(a)_i = \frac{\exp\left(\sum_{k \in [m]}a_k H_{ik}\right) }{ \sum_{j \in [m]}a_j \exp\left(\sum_{k \in [m]}a_kH_{jk}\right)}, \;\; \forall i: a_i = 1.\end{equation}
If $H$ is a diagonal matrix, then this is a standard MNL with parameters taken from the diagonal entries of $H$. A non-zero off-diagonal entry $H_{ij}$ is meant to represent the halo effect that the presence of product $j$ has on the choice probability of product $i$. These pairwise effects can be both positive ($H_{ij} > 0$) and negative ($H_{ij} < 0$). The Halo MNL can be represented as a fully-connected neural network with no hidden layers. 
As an aside, the Halo MNL is {\em not} a generalization of the Mixture MNL, though a mixture Halo MNL would be (we will not consider this here).

%% file: sections/section_our_model_2_lenses.tex
\section{Our Model Through Two Lenses}
In light of the previous discussion of existing choice models,  there are two ways to motivate the choice model we propose. While these motivations are completely orthogonal to each other, it will turn out that they are in fact nicely connected.
\begin{enumerate}
    \item 
The first motivation is informal/empirical, and follows from the reaction that someone well-versed in deep learning likely has when finding out that the neural networks described in the \cref{chapter3} represent the state-of-the-art. That reaction is that the networks are fairly primitive, particularly with respect to other application areas for which more-complex network architectures are successful.

In designing a new choice model, a natural proposal then would be to explore more-sophisticated network architectures, even if the resulting choice model no longer lends itself to nice interpretations. The most immediate option is simply to explore depth (i.e.~more layers) -- this is something we try in our benchmark, and ultimately find does {\em not} work well empirically. Another option is {\em self-attention}, which works incredibly well in other applications, and is naturally well-suited to choice modeling. In particular, the motivation of self-attention is to capture effects that are similar, at least in spirit, to halo effects. Moreover, since we are working with subsets here rather than sequences, there is less need for the ad-hoc engineering (e.g.~positional encodings) required for other applications. All of this is to say that we might a priori expect a simple choice model built on self-attention to demonstrate empirical success. This is precisely the model we propose, and we do indeed see such empirical success.
\item 
The second motivation is more principled and is completely orthogonal to the first. Consider again the Halo MNL choice model. This is the most-recent, and thus least-established, of the three choice models we defined, but there is ample experimental evidence from the behavioral sciences literature \citep{thorndike1920constant, simonson1992choice, feng2018substitutability} that halo effects do exist. The primary problem with this model is that the amount of data needed to accurately estimate its parameters is too large. In particular, it suffers from $\Omega(m^2)$ sample complexity:
\input{theorems/thm_halo_mnl_sample_complexity}

\cref{prop:LB-HaloMNL} shows that achieving a non-trivial mean absolute error when estimating the parameter matrix $H$ requires the number of transactions to at least $\Omega(m^2)$.
At a high level, we still would like to capture halo effects, but need to do so in a manner which allows for accurate model estimation using, ideally, $O(m)$ transactions. 

Thus motivated, our proposed model is essentially a low-rank version of the Halo MNL. Specifically, we define the {\bf Low-Rank Halo MNL} to be the choice model in Eq.~\ref{eqn:HaloMNL}, with the additional assumption that
the matrix $H$ additively decomposes into a diagonal matrix and a low-rank matrix:
\begin{equation} \label{eqn:LowRankHalo} H = \mathrm{diag}(\alpha) + UV^\top, \end{equation}
where $\alpha \in \mathbb{R}^m$, the $\mathrm{diag}(\dot)$ operator maps a vector to its corresponding diagonal matrix, and $U,V \in \mathbb{R}^{m \times r}$ for some {\em rank} $r \in [m]$. 

The Low-Rank Halo-MNL can be thought of as interpolating between the classic MNL and the Halo MNL. Taking $r$ to be $m$ would return the ``full'' Halo MNL model, and loosely speaking, $r=0$ would be the standard MNL (this is the reason we include $\alpha$, rather than simply modeling $H$ as low-rank, as even the simple MNL  already yields a full-rank matrix;  here only the halo effects are taken to be low-rank). In general, smaller values of $r$ yield more structure and ideally lower sample complexity. In fact, we might expect that $O(rm)$ transactions suffices, as in the following lower bound:

\input{theorems/thm_low_rank_halo_mnl_sample_complexity}

\end{enumerate}

Returning to our first motivation, the two are intimately linked. Precisely, the Low-Rank Halo MNL is in fact {\em equivalent} to a network with the self-attention mechanism.
\input{theorems/prop3withproof}


%% file: theorems/thm_halo_mnl_sample_complexity.tex
\begin{theoremEnd}[category=halomnlsamplecomplexity]{proposition}[Halo MNL Sample Complexity]
\label{prop:LB-HaloMNL}
    For any Halo MNL estimator $\hat{H}$, and any set of $n$ assortments, there exists a ground truth $H^*$ such that
    \[ \frac{1}{m^2}\mathbb{E}\|  \hat{H}  - H^* \|_1 = \Omega \left( \frac{m^2}{n} \right).  \]
\end{theoremEnd}
\begin{proofsketch}
    Here, we use $\| \cdot\|_1$ to denote the sum of the absolute values of the entries. We apply a Fano/Assouad-type analysis on the family of Halo MNL choice functions gotten by setting the off-diagonals of $H^*$ to be $\pm\delta$ for some properly-chosen $\delta$, and the diagonals to be $0$.
\end{proofsketch}
\begin{proofEnd}
Fix any assortment $a \in \{0,1\}^m$, and let $H,H' \in \mathbb{R}^{m \times m}$ be two Halo MNL parameter matrices which differ on a single entry by an additive factor $\delta$, i.e. $H' = H + \delta e_ie_j^\top$ for some $i,j \in [m]$. The KL-divergence of the random choices made under these two models under assortment $a$, which we denote by $\mathcal{K}_a(H||H')$, can be upper bounded as follows:
\begin{align*}
\mathcal{K}_a(H||H') 
&= \sum_{i \in a} \frac{\exp(Ha)_i}{\sum_{j \in a} \exp(Ha)_j }\log\left(\frac{\exp(Ha)_i}{\exp(H'a)_i} \frac{\sum_{j \in a}\exp(H'a)_j}{\sum_{j \in a}\exp(Ha)_j} \right) \\
&=   \frac{\sum_{i \in a} \exp(Ha)_i  (Ha-H'a)_i}{\sum_{j \in a} \exp(Ha)_j}
+ \log\left( \frac{\sum_{j \in a}\exp(H'a)_j}{\sum_{j \in a}\exp(Ha)_j} \right) \\
&\le \max_{i \in a}|(Ha-H'a)_i| 
+ \log \left(1 + \delta \frac{\exp(H'a)i}{\sum_{j \in a}\exp(Ha)_j}\right) \\
&\le \max_{i \in a}|(Ha-H'a)_i| 
+ \log \left(1 + \delta \right) \\
&\le 2\delta.
\end{align*}
Since the final expression does not depend on the assortment $a$, and the KL-divergence tensorizes over product measures, we have that the KL-divergence of the choices made across all $n$ assortments can be bounded as 
\[  \mathcal{K}(H || H') \le 2n\delta. \]

Now consider two parameter matrices $H,H'$ whose diagonal entries are zero, and whose off-diagonal entries belong to $\{\delta,-\delta\}$. Then iteratively applying the above yields
\[ \mathcal{K}(H || H') \le 2n\delta d_\mathrm{ham}(H,H'),  \]
where $d_\mathrm{ham}(H,H')$ is the number of entries on which $H$ and $H'$ differ.

Let $k = \frac{1}{2}(m^2-m) - m^{1+\epsilon}$, for any $\epsilon > 0$.
The packing number of the Hamming cube $\{\delta,-\delta\}^{m^2-m}$ is sufficiently large so that there exist at least $N = O(2^{m^2-m}/m^{2k})$ such parameter matrices whose pairwise hamming distance is at least $k$. By Assouad's lemma, we conclude that for any estimator $\hat{H}$, there exists at least one of these parameter matrices such that 
\[  \mathbb{E}\|\hat{H}-H^*\|_1 \ge \frac{\delta k}{2}\left( 1 - \frac{2n\delta k + \log 2}{\log N}\right) = \Omega\left(\delta k \left(1 - \frac{2n\delta k + \log 2}{m^2 - 2k \log m} \right) \right).\]
Taking $\delta$ to be constant, and $k \sim m^2/n$ completes the proof.
\end{proofEnd}

%% file: theorems/thm_low_rank_halo_mnl_sample_complexity.tex
\begin{theoremEnd}[category=lowrankhalomnlsamplecomplexity]{proposition}[Low-Rank Halo MNL Sample Complexity]
\label{prop:LB-LowRankHalo}
    For any Low-Rank Halo MNL estimator $\hat{H}$, and any set of $n$ assortments, there exists a ground truth $H^*$ such that
    \[ \frac{1}{m^2}\mathbb{E}\|  \hat{H}  - H^* \|_1 = \Omega \left( \frac{rm}{n} \right).  \]
\end{theoremEnd}
\begin{proofsketch}
    Similar to the proof of Proposition \ref{prop:LB-HaloMNL}, we apply a Fano/Assouad-type argument, this time on the family of Halo MNL matrices $H^*$ with $r$ columns of $\pm \delta$ entries.
\end{proofsketch}
\begin{proofEnd}
Consider the set of parameter matrices which are non-zero only on their first $r$ columns, whose diagonal entries are zero, and the non-zero entries belong to $\{\delta,-\delta\}$. The entire proof of Proposition \ref{prop:LB-HaloMNL} can be repeated, though the corresponding Hamming cube is on $O(rm)$ dimensions. The result follows directly.
\end{proofEnd}

%% file: theorems/prop3withproof.tex
\begin{proposition}
\label{prop:LowRank-Attention}
    The Low-Rank Halo MNL choice model can be represented as a neural network with a single self-attention head.
\end{proposition}
\begin{proof}
Recall that the Low-Rank Halo MNL is encoded by $H = \mathrm{diag}(\alpha) + UV^\top$ for $\alpha \in \mathbb{R}^m$ and $U,V \in \mathbb{R}^{m \times r}$. It takes as input an assortment $a \in \{0,1\}^m$, and outputs choice probabilities $p(a) \in \Delta^{m-1}$ according to:
\[ p(a)_i = \frac{a_i \exp(Ha)_i}{a^\top \exp(Ha)},  \]
or equivalently, if the assortments are encoded as $a \in \{0,e\}^m$, then 
\[ p(a) = \mathrm{softmax}(e\log(a) \circ (Ha)).\]

Now consider a neural network with three layers:
\begin{enumerate}
\item {\bf Input:} The input layer consists of $m$ nodes. Note that this is essentially just a tokenization of the products, and any assortment $a \in \{0,1\}^m$ can be input.
\item {\bf Self-attention:} A single self-attention head will compute so-called queries $Q \in \mathbb{R}^{m \times d}$ and keys $K \in \mathbb{R}^{m \times d}$ for some $d \in [m]$. In this case, take $d = r$, and let $Q = U$, and $K = \sqrt{r}V$. The self-attention layer will also compute values $Y \in \mathbb{R}^{m \times d}$ for some $d \in [m]$. In this case, take $d = m$, and let $Y = I$, the identity matrix. The result of this layer, after feeding through an assortment $a$ from the input layer, is $m$ outputs of the form $\frac{QK^\top Va}{\sqrt{r}}$, which is equal to $UV^\top a$.\footnote{Note that there is often a softmax applied to $QK^\top$. We omit this here.}
\item {\bf Output:} The output layer consists of $m$ nodes, and is gotten by taking the softmax of the output of the self-attention head, along with single edges from the input layer with logarithmic activation. Setting the weights on the single edges to be $\exp(\alpha)$, the output is a softmax applied to $e\log(a) \circ (Ha)$ as was needed.
\end{enumerate}
\end{proof}

%% file: sections/section_estimation_theory.tex
\section{Estimation and Theoretical Guarantees}
Recall that choice data is in the form of transactions $(a_1,y_1),\ldots,(a_n,y_n)$, where each $a_i \in \{0,1\}^m$ is an assortment, and each $y_i$ is a unit-vector containing a single $1$ representing the choice made. Given a set of choice data, the {\em likelihood function} for the Halo MNL can be written compactly as
\[ \mathcal{L}(H) = \prod_{i=1}^n \frac{y^\top \exp(Ha)}{a^\top \exp(Ha)},\]
where we are using $\exp(\cdot)$ to denote the entry-wise exponential.

We propose to estimate the Low-Rank Halo MNL via a simple nonconvex estimator: 

\begin{equation} \label{eqn:nonconvex}
\argmax_{\alpha,U,V} \log \mathcal{L}(\mathrm{diag}(\alpha) + UV^\top) + \lambda(\|\alpha\|^2 + \|U\|_F^2 + \|V\|_F^2), \end{equation}
where $\lambda \ge 0$ is a tuning parameter, and $\|\cdot\|_F$ denotes the matrix Frobenius norm. For $\lambda = 0$, this is the estimator that results from training a standard neural network implementation of the Low-Rank Halo MNL with the cross-entropy loss. Taking $\lambda > 0$ introduces regularization, which (a) solves an identifiability issue (otherwise the parameters are only unique up to a multiplicative factor), and (b) yields stronger convergence guarantees.

In the spirit of demonstrating that the Low-Rank Halo MNL has $O(rm)$ sample complexity, rather than the $O(m^2)$ required for the Halo MNL, we show that under a reasonable generative model of the assortments, the objective function in Eq.~\ref{eqn:nonconvex} is likely to contain a stationary point ``close'' to the true parameters:
\input{theorems/thm_estimation}

Theorem \ref{thm:UB} does not preclude the existence of other stationary points, so part of the value of our coming experiments is to demonstrate that the Low-Rank Halo MNL can be practically estimated. 

%% file: theorems/thm_estimation.tex
\begin{theoremEnd}[category=lowrankestimation]{theorem}\label{thm:UB}
    Let $\mathcal{A}$ denote a random distribution taking values on $\{0,1\}^m$ such that for some $p \in (0,1/2)$,
    \begin{enumerate}
        \item $\mathbb{P}(\mathcal{A}_j = 1) \in [p,1-p]$ for all $j \in [m]$
        \item $\mathbb{P}(\mathcal{A}_j = 1 | \mathcal{A}_{j'} = 1 ) \in [p,1-p]$ for all $j,j' \in [m]$ such that $j \ne j'$.
    \end{enumerate}
    Let $n = \Omega(rm\log m)$, and
suppose assortments $a_1,\ldots,a_n$ are drawn independently from $\mathcal{A}$. Then with probability at least $1 - o(1)$, the objective function in Eq.~\ref{eqn:nonconvex} has a local maximum $(\hat{\alpha},\hat{U},\hat{V})$ satisfying
\[  \frac{1}{m^2}\|\hat{H} - H^*\|_F^2 = o(1), \]
where $\hat{H} = \mathrm{diag}(\hat{\alpha}) + \hat{U} \hat{V}^\top$.
\end{theoremEnd}
\begin{proofEnd}
First note that given our conditions on $\mathcal{A}$, and the fact that we will only need to show our result up to constants (which may depend on $\mathcal{A}$), it suffices to assume that $\mathcal{A}$ is in fact the distribution which includes each product independently with probability $1/2$ (since our bounds would depend continuously on $p$ for $p > 0$). Along the same lines, it also suffices to assume that $H^*$ is such that the resulting choice probabilities are always uniform; we could perform a more fine-grained analysis by defining \[p^*_{\max} = \max_{a \in \{0,1\}^m} \max_{i \in [m]} \frac{a_i \exp(Ha)_i}{a^\top \exp(Ha)}, \] and $p^*_{\min}$ equivalently, in which case our result would depend continuously on their ratio.

Now fix a single transaction $(a,y)$, and let $\ell(H)$ denote the corresponding log-likelihood function:
\[  \ell(H) = \log \frac{y^\top \exp(Ha)}{a^\top \exp(Ha)} = y^\top H a - \log (a^\top \exp(Ha)). \]
For convenience of notation, define the vector $p$ to be $p = \exp(Ha)/a^\top \exp(Ha)$.

Then the gradient of $\ell(\cdot)$ is 
\[ \nabla_H \ell(H) =  ya^\top - \frac{\exp(Ha)}{a^\top \exp(Ha)}a^\top = (y-p)a^\top,  \]
and the second derivatives are 
\[ \nabla_{H_{Ij}}\nabla_{H_{I\ell}}\ell(H) = -a_\ell a_j \exp(Ha)_I \frac{a^\top \exp(Ha)   - \exp(Ha)_I   }{(a^\top \exp(Ha))^2} = -a_j a_\ell  p_I + a_j a_\ell p_I^2, \]

and for $I \ne I'$:
\[ \nabla_{H_{Ij}}\nabla_{H_{I'\ell}}\ell(H) = a_j a_\ell\frac{\exp(Ha)_I  \exp(Ha)_{I'}  }{(a^\top \exp(Ha))^2} = a_j a_\ell p_I p_{I'}. \]

Now consider a parameterization of $H$ by some scalar $x$. Let $\nabla_x H$ denote the $m \times m$ matrix whose entries are $\nabla_x H(x)_{ij}$. Then applying the chain rule gives the following: 
\[ \nabla_x \ell(H) = \langle \nabla_H \ell(H), \nabla_x H \rangle = (y-p)^\top (\nabla_x H)a. \]
Letting $\nabla_x^2 H$ denote the matrix whose entries are $\nabla_x^2 H(x)_{ij}$, we can compute the second derivative as well:
\begin{align*}
\nabla_x^2 \ell(H) 
&= \langle \nabla_x^2 H, \nabla_H \ell(H) \rangle + \sum_{i,j=1}^m \sum_{i',j'=1}^m \nabla_x H(x)_{ij} \nabla_x H(x)_{i'j'} \nabla^2_{H_{ij}H_{i'j'}}\ell(H) \\
&= (y-p)^\top (\nabla^2_x H)a + \left(p^\top (\nabla_x H)a \right)^2 - a^\top (\nabla_x H)^\top \mathrm{diag}(p) (\nabla_x H) a.
\end{align*}
Consider these terms separately in order:
\begin{enumerate}
    \item The first term is mean-zero (because $\mathbb{E}[y-p|a] = 0$), 
    and bounded in magnitude by $\|\nabla_x^2 H a\|_\infty.$ By Hoeffding's inequality, when summing over $n$ transactions, we have that with probability $1 - o_n(1)$,
    \[\sum_{i=1}^n (y_i-p_i)^\top (\nabla^2_x H)a_i \le \omega_n(1)\sqrt{n} \max_a\|(\nabla^2_x H) a\|_\infty\]
    
    \item The second term has mean \[ \mathbb{E}\left(p^\top (\nabla_x H)a \right)^2 \approx \frac{(e^\top (\nabla_x H)e)^2}{m^2}, \]
    and by Hoeffding's inequality, with probability $1-o_m(1)$, it holds that \[ \left(p^\top (\nabla_x H)a \right)^2 \le \omega_m(1) \frac{(e^\top (\nabla_x H)e)^2}{m^2}. \]
    Summing over $n$ transactions, we have that with probability $1 - no_m(1)$,
    \[\sum_{i=1}^n  \left(p_i^\top (\nabla_x H)a_i \right)^2 \le n\omega_m(1) \frac{(e^\top (\nabla_x H)e)^2}{m^2}.  \]
    
    \item The third term has mean \[ \mathbb{E}\left[a^\top (\nabla_x H)^\top \mathrm{diag}(p) (\nabla_x H) a\right] \approx  \frac{e^\top (\nabla_x H)^\top (\nabla_x H) e}{m}, \]
    and with probability $1 - o_m(1)$,
    \[ a^\top (\nabla_x H)^\top \mathrm{diag}(p) (\nabla_x H) a \ge o_m(1)\left(\frac{e^\top (\nabla_x H)^\top (\nabla_x H) e}{m} + \frac{\| \nabla_x H \|_F^2}{m} \right).  \]
    Thus, with probability $1 - no_m(1)$,
    \[ \sum_{i=1} ^n a_i^\top (\nabla_x H)^\top \mathrm{diag}(p_i) (\nabla_x H) a_i \ge no_m(1)\left(\frac{e^\top (\nabla_x H)^\top (\nabla_x H) e}{m} + \frac{\| \nabla_x H \|_F^2}{m} \right). \]
\end{enumerate}

Now consider $H$ parameterized by scalar $x$ in the following way:
\[ H(x) = \mathrm{diag(\alpha^* + x\delta_\alpha)} + (U^* + x\delta_U)(V^* + x\delta_V)^\top,\]
where $\delta$ (which is the vector-concatenation of $\delta_\alpha, \delta_U,$ and $\delta_V$) is a norm-1 vector. We have that
\[  \nabla_x H = \mathrm{diag}(\delta_\alpha) + \delta_UV^{*\top} + U^* \delta_V^\top + 2x \delta_U \delta_V^\top, \]
and
\[ \nabla_x^2 H = 2\delta_U \delta_V^\top. \]
Observe that, for $x=0$, the function $\delta \to \nabla_x H$ is a linear mapping. Let $T^\perp$ denote the nullspace of this linear map, and $T$ denote the subspace perpindicular to $T^\perp$. 

There exists some constant $\gamma > 0$ such that for any norm-1 $\delta \in T$, the matrix $\nabla_x H$ contains at least one entry  larger in magnitude than $\gamma$. Combining this with the analysis of the three terms above, we conclude that with probability $1 - no_m(1) - o_n(1)$, at $x=0$, \[ \nabla_x^2 \ell(H) \le C\frac{no_m(1) \gamma^2}{m}, \]
for some constant $C > 0$.

To summarize what we have shown so far, we have that our estimator's objective function, namely
\[ f(\alpha,U,V) = \log \mathcal{L}(\mathrm{diag}(\alpha) + UV^\top) - \lambda (\|\alpha\|^2 + \|U\|_F^2 + \|V\|_F^2),  \]
is, with probability $1 - no_m(1) - o_n(1)$, strongly concave at $(\alpha^*,U^*,V^*)$ with parameter \[M = \min\left\{ C\frac{no_m(1) \gamma^2}{m},\lambda\right\}.\]

It remains to bound the gradient of $f$ at $H^*$:
\[ \| \nabla_{(\alpha,U,V)} f(\alpha^*,U^*,V^*) \| \le \left\| \sum_{i=1}^n \nabla_{(\alpha,U,V)} \ell(H^*) \right\| + 2\lambda(\|\alpha^*\| + \|U^*\|_F + \|V^*\|_F).  \]
The second term is of size $\lambda O(\sqrt{rm})$. The first term can be bounded as
\begin{align*} \left\| \sum_{i=1}^n \nabla_{(\alpha,U,V)} \ell(H^*) \right\| 
\le \left\| \sum_{i=1}^n (y_i - p_i)^\top a_i \right\| + \left\| \sum_{i=1}^n (y_i - p_i) a_i^\top V^* \right\| +  \left\| \sum_{i=1}^n U^{*^\top} (y_i - p_i) a_i^\top  \right\|,
\end{align*} which with probability $1 - o_n(1)$ is bounded by $\omega_n(1) \sqrt{nrm}. $

To conclude, there exists a stationary point $(\hat{\alpha},\hat{U},\hat{V})$ such that 
\[ \frac{\|(\hat{\alpha},\hat{U},\hat{V}) - (\alpha^*,U^*,V^*)\|^2}{rm} \le \frac{\omega_n(1)rm}{n}, \]
and the result follows.
\end{proofEnd}

%% file: sections/section_experiments.tex
\section{Experiment}
\begin{algorithm}[tb]
\caption{Low-Rank Halo MNL. $\ell$ is MNL layer for diagonal values in attention matrix $A$, denoted as $M$, and $\zeta$ is a function to generate query, key, and value: $Q, K \in \mathbb{R}^{m \times d_k}, V\in\mathrm{R}^{m \times d_v}$, where $d_k$ and $d_v$ is the dimension of key and value.}
\label{alg:alg_attention}
\begin{algorithmic}[1]
\State $M \gets \ell(S)$    \Comment{Calculate MNL values}
\State $Q, K, V \gets \zeta(S)$ \Comment{Get query, key, and value}
\State $A \gets \frac{Q \cdot K^T}{\sqrt{d_k}}$ \Comment{Calculate attention}
\If{Normalization}
    \State $A \gets \text{softmax}(A)$
\EndIf 
\State $A \gets A \cdot V$  
\State $A_{\text{diag}} \gets M$    \Comment{Replace diagonal values with $M$}
\State $p(S) \gets \sum_{j} A_{i,j}$ 
\end{algorithmic}
\end{algorithm}

\begin{table*}[t]
\caption{Summary of experimental results on Hotel Dataset}
\label{tab:hotel_table}
\vskip 0.15in
\begin{center}
\begin{small}
\begin{sc}
\begin{tabular}{cccccc}
\toprule
\textbf{Model} &
Hotel \#1 &
Hotel \#2 & 
Hotel \#3 &
Hotel \#4 &
Hotel \#5 \\ 
\midrule
\textbf{Our Model} &
 \textbf{0.7872}&  \textbf{0.7535}& 0.7154  &  \textbf{0.6922}&0.7894
\\ 
MNL & 0.8320
 &0.7966& 0.7423 & 0.7459 & 0.8118\\
Mixed MNL & 0.8337
 &0.7924  & 0.7411 & 0.7019 &0.8112\\
Halo MNL &0.7874
 & 0.7558 & 0.7136 & 0.6995 &0.7971\\
Exponomial &0.8070
 &  0.7593&0.7173  & 0.7199 &0.7904\\
Markov Chain & 0.7884
 & 0.7536 & \textbf{0.7123} & 0.7306 &\textbf{0.7887}\\
Stochastic Preference & 0.7967
 & 0.7574 & 0.7167 & 0.7368 & 0.7926\\
 \bottomrule
\end{tabular}
\end{sc}
\end{small}
\end{center}
\vskip -0.1in
\end{table*}

In this section, we proceed to empirically validate the Low-Rank Halo MNL model, aiming to demonstrate its practical estimability and to measure its predictive performance against a suite of established choice modeling methods. Through this rigorous comparative analysis, we intend to solidify the model's standing both theoretically and in real-world applications, highlighting its capacity to capture consumer choice dynamics effectively. 

Typically, three principal datasets serve as benchmarks in this field: the sushi dataset \citep{kamishima2003nantonac}, the hotel dataset \citep{bodea2009data}, and the IRI Academic dataset \citep{bronnenberg2008database}. However, the sushi dataset, primarily comprising preference lists from individual customers, primarily reflects \textit{rational} choice aspects. This limitation renders it less suitable for our experiments. Consequently, our analysis will focus on utilizing the hotel and the IRI Academic datasets to assess and quantify the predictive capabilities of the models under study.

\subsection{Dataset Description}

\paragraph{Hotel Dataset}
The hotel dataset encompasses a collection of real-world booking transactions primarily from business clients at five hotels within the continental U.S. These bookings, recorded between March 12, 2007, and April 15, 2007, feature a minimum booking horizon of four weeks for each check-in date. The dataset captures detailed information at each booking instance, including the rate and availability of different room types. This data was meticulously gathered from various sources, including direct reservations through the hotel or customer relationship officers, online bookings via the hotels' websites, and transactions made through offline travel agencies. The preprocessing of the data was conducted in alignment with the methodologies outlined in prior research \citep{van2015market, csimcsek2018expectation, berbeglia2022comparative}. After preprocessing, the product assortment in each hotel is refined to include between 4 and 10 distinct room types. In terms of scale, the hotel dataset comprises a modest volume, spanning \textit{several thousand} entries, with the largest subset approaching 6,000 records. This size is notably smaller in comparison to the more extensive IRI Academic Dataset that follows.

\paragraph{IRI Academic Dataset}
The IRI dataset consists of real-world consumer purchase transactions for 31 product categories, collected from 47 markets across the United States, and is a main benchmark in the choice modeling field \citep{chen2022decision, jena2022estimation, jagabathula2019limit}. To ensure a fair comparison with previous works, we follow the data preprocessing steps outlined by \citet{jagabathula2019limit} for our experiment. Items within the dataset are marked with their corresponding Universal Product Code (UPC). By identifying and grouping together items that share the same vendor code, specifically the digits 4-8 of the UPC, we categorize them as a single product. This method facilitates the distinction of individual products from unprocessed transaction data. 

In contrast to some previous studies that limited experiments to the top 10 products, we extend our study to the top 20 products, which we think aligns more closely with real-world scenarios. By selecting the top 19 purchased products in the dataset and treating the remaining products as a no-purchase option, we create the transaction records for the experiment. In alignment with the work by \citet{jagabathula2019limit}, we focus on the data from the calendar year 2007 and compile two datasets (short and long-term) containing the first 4 weeks of data and the entire 52 weeks of data in 2007, respectively. 

Contrasting with the more constrained hotel dataset and previous IRI dataset implementations (which typically involved assortments of up to 10 products and a data span of 2 weeks), our expanded real-world scale IRI dataset encompasses a significantly larger scope. It includes \textit{several million} transactions per product category, with the most extensive category comprising around 17 million assortments. This marks a substantial leap from previous IRI research setups, which generally utilized datasets of only a few hundred thousand entries. This expansive scale of data not only underscores the real-world applicability of our study but also sets a new benchmark for comprehensive and realistic experiments in the field of choice modeling.
\subsection{Benchmark Models}

\begin{figure*}[t]
\centering
\includegraphics[width=0.9\textwidth]{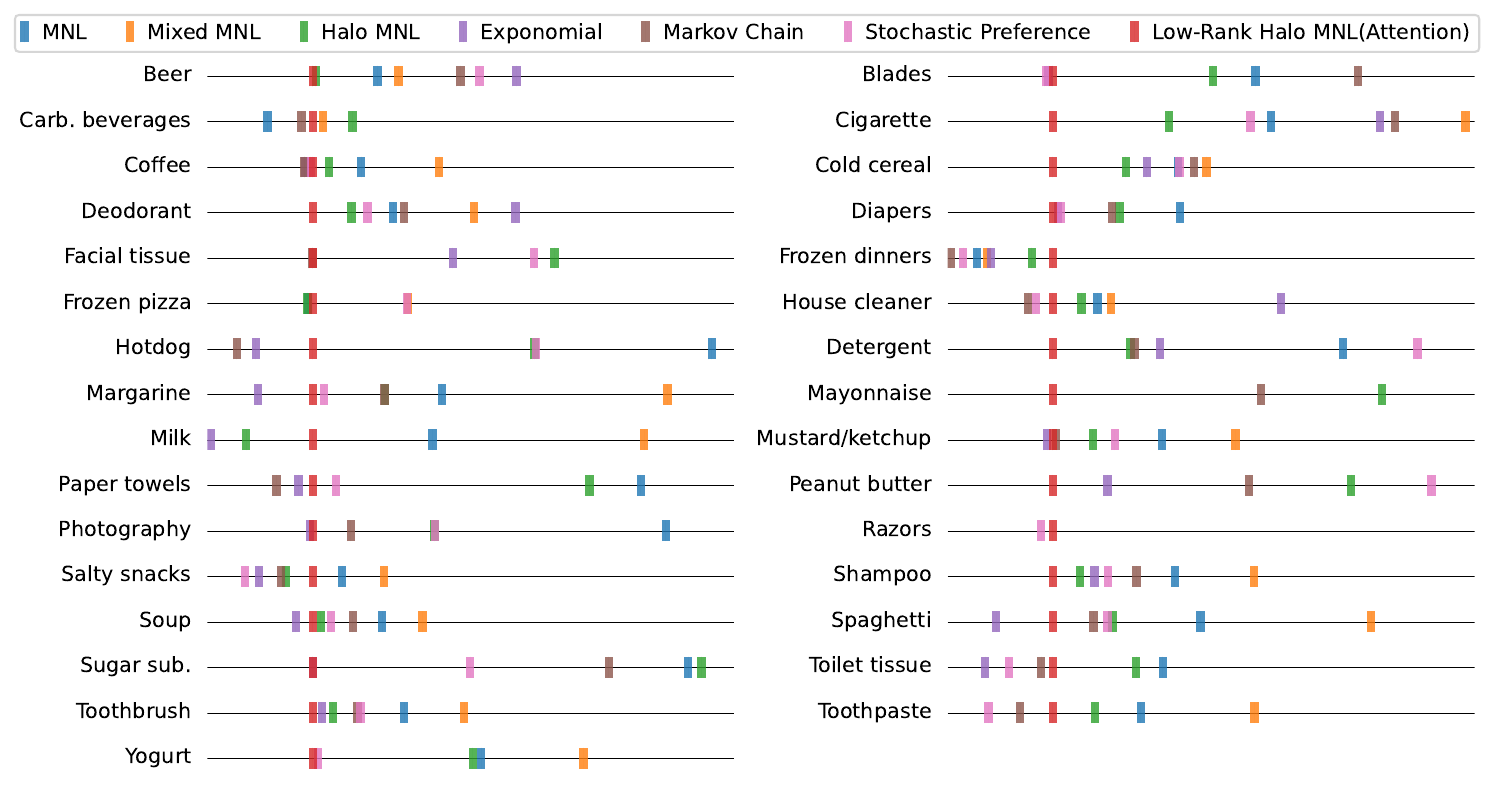} 
\caption{Detailed breakdown of experimental results on 4 weeks of data. Cross-entropy loss is reported across product categories, normalized to our model (Low-Rank Halo MNL).}
\label{4weeks}
\end{figure*}

Our experiments extend beyond the traditional MNL, Mixed MNL, and Halo MNL models, embracing a wide range of choice models, including recent applications employing deep learning frameworks. The array of models we investigate encompasses the Exponomial model \citep{daganzo2014multinomial, alptekinouglu2016exponomial}, Markov Chain Model \citep{blanchet2016markov}, Stochastic Preference model \citep{mahajan2001stocking}, RUMnet \citep{aouad2022representing}, Decision Forest \citep{chen2022decision}, and several other deep learning architectures adapted for choice models including the ones from \cite{cai2022deep}.

Of particular note are the three top-performing algorithms in our experimental design: the Exponomial model, the Markov Chain model, and the Stochastic Preference model. These have consistently showcased superior performance in previous comparative research \citep{berbeglia2022comparative}, distinguishing themselves among various competing models. Specifically, the Exponomial model has excelled with large in-sample data, the Markov Chain model has proven superior with small in-sample data, and the Stochastic Preference model has demonstrated strong overall performance throughout the study.

\subsection{Experiment Details}

For the hotel dataset, we adhered to the train-test split ratio used in previous studies to ensure a fair comparison. Specifically, we randomly allocated 80\% of the transactions for training and reserved 20\% for testing. In the case of the IRI Academic dataset, the experimental setup was slightly different. To conduct both short-term and long-term analyses effectively, we allocated 70\% of the data for training purposes, following the precedent set by prior research. The remaining 30\% of the data was used for testing to comprehensively evaluate the performance of the models.

To ensure equitable comparisons, we adhere to the implementation of Exponomial, Markov Chain, and Stochastic Preference models as described by \citet{berbeglia2022comparative}. The MNL, Mixed MNL, and Halo MNL models are formulated as neural networks, allowing the models' parameters to be estimated via the neural network's weight parameters. Our proposed Low-Rank Halo MNL is also parameterized as a neural network, employing a self-attention mechanism as outlined in Algorithm \ref{alg:alg_attention}. Consistently for all models, we use cross-entropy loss as the performance metric, and we report all results with averages over three distinct random seeds to ensure the reproducibility and reliability of our training process.

\subsection{Results}

\begin{figure*}[t]
\centering
\includegraphics[width=0.9\textwidth]{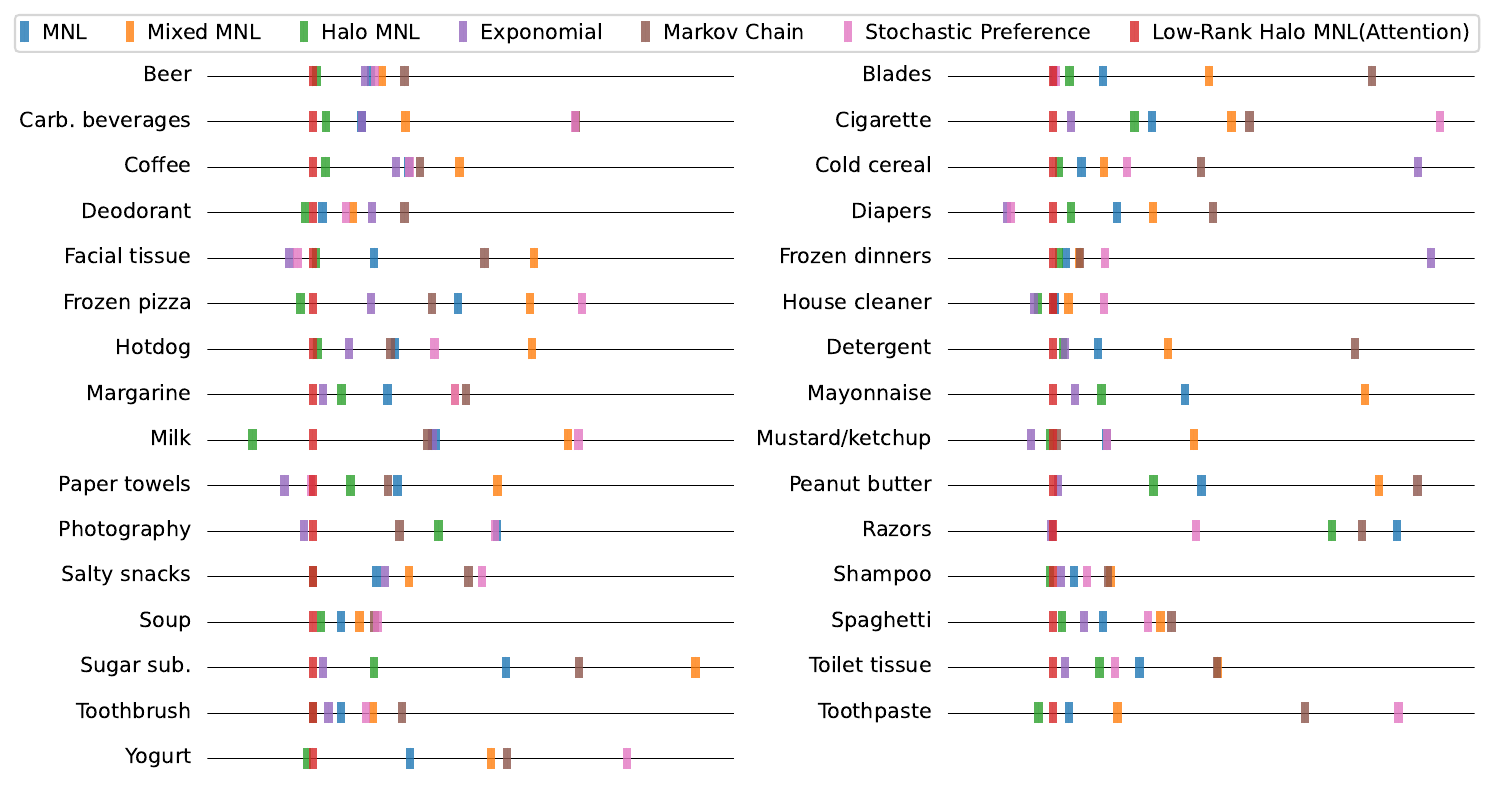} 
\caption{Detailed breakdown of experimental results on 52 weeks of data. Cross-entropy loss is reported across product categories, normalized to our model (Low-Rank Halo MNL).}
\label{52weeks}
\end{figure*}

\input{sections/table_summary_4weeks}

The experimental findings for the hotel dataset are concisely presented in \cref{tab:hotel_table}. In hotels 1, 2, and 4, our Low-Rank Halo MNL model demonstrated superior performance compared to other models. Meanwhile, in hotels 3 and 5, the Markov chain model showed better results. However, given the relatively small size of this dataset, the challenge of estimation using any choice model that accounts for irrational behavior is presumably less daunting. Nonetheless, our experiments affirm the successful implementation and estimation of the Low-Rank Halo MNL within this dataset context.

For the IRI Academic dataset, the results are thoroughly outlined in \cref{tab:4weeks_table} and \cref{tab:52weeks_table}, and are visually represented in \cref{4weeks} and \cref{52weeks}. These figures depict the cross-entropy loss for each product category over both 4-week and 52-week periods. Each line in the graph is normalized against the loss value of our model (Low-Rank Halo MNL), anchoring at a consistent point across all lines. The ends of each line represent 98\% and 108\% of our model's loss value, with a position leaning towards the left indicating a superior performance in that particular category.

\input{sections/table_summary_52weeks}

Building on this visual representation, \cref{tab:4weeks_table} and \cref{tab:52weeks_table} quantify these relationships, providing a comprehensive performance overview across product categories. The ‘Number of Wins’ column enumerates instances where each model outperforms others in different product categories. 'Relative Loss' reflects an average, normalized against the mean loss of the Low-Rank Halo MNL, where a lower figure indicates a more favorable overall performance. In the 4-week dataset, our model leads in 11 out of 31 categories, while the Exponomial model is also noteworthy, leading in 10 categories. For the 52-week dataset, our model excels in 18 of the 31 categories, underscoring its enhanced performance in longer-term analyses.

The long-term (52-week) data reveals a particularly interesting pattern where our model prevails in a significant majority of categories. This trend may suggest that the Low-Rank Halo MNL, with its unique self-attention mechanism, is especially capable of capturing the nuances of choice behavior in datasets featuring a broader spectrum of transactions and assortments. Its improved performance in the long-term dataset could be indicative of the model's proficiency in balancing rational and irrational elements of consumer choice behavior.

%% file: sections/table_summary_4weeks.tex
\begin{table}[b!]
\caption{Summary of experimental results on 4 weeks of data.}
\label{tab:4weeks_table}
\vskip 0.15in
\begin{center}
\begin{small}
\begin{sc}
\begin{tabular}{lcccr}
\toprule
\textbf{Model} &
\textbf{Wins ($\uparrow$)} &
\textbf{Loss ($\downarrow$)} & \\ 
\midrule
  \textbf{Our Model} &
  \textbf{11} &
  \textbf{0.0\%}\\ 
MNL &
  0 &
  3.64\%\\
  Mixed MNL &
  0 &
  8.89\%\\
  Halo MNL &
  0 &
  2.19\%\\
  Exponomial &
  10 &
  0.07\%\\
  Markov Chain &
  8 &
  0.66\%\\
  Stochastic Preference &
  2 &
  1.67\%\\
\bottomrule
\end{tabular}
\end{sc}
\end{small}
\end{center}
\vskip -0.1in
\end{table}

%% file: sections/table_summary_52weeks.tex
\begin{table}[b!]
\caption{Summary of experimental results on 52 weeks of data.}
\label{tab:52weeks_table}
\vskip 0.15in
\begin{center}
\begin{small}
\begin{sc}
\begin{tabular}{lcccr}
\toprule
\textbf{Model} &
\textbf{Wins ($\uparrow$)} &
\textbf{Loss ($\downarrow$)} & \\ 
\midrule
\textbf{Our Model} &
\textbf{18} &
\textbf{0.0\%}\\ 
MNL &
0 &
1.54\%\\
Mixed MNL &
0 &
3.49\%\\
Halo MNL &
6 &
0.48\%\\
Exponomial &
7 &
3.88\%\\
Markov Chain &
0 &
3.05\%\\
Stochastic Preference &
0 &
3.54\%\\\bottomrule
\end{tabular}
\end{sc}
\end{small}
\end{center}
\vskip -0.1in
\end{table}

%% file: sections/section_conclusion.tex
\section{Conclusion}
Motivated by the lack of modern deep learning techniques in the application area of choice model estimation, we proposed a  choice model that leverages a modern neural network architectural concept (self-attention). We showed that our attention-based choice model is a low-rank generalization of the Halo Multinomial Logit (Halo-MNL) model and proved that whereas the Halo-MNL requires $\Omega(m^2)$ data samples to estimate, our model supports a natural nonconvex estimator (in particular, that which a standard neural network implementation would apply) which admits a near-optimal stationary point with $O(m)$ samples. We then established the first realistic-scale benchmark for choice estimation on real data and used this benchmark to run the largest evaluation of existing choice models to date. Our findings demonstrate that our model consistently excels in both conventional experimental settings and our more expansive experimental frameworks.

%% file: sections/section_appendix_experiments.tex
\section{Details on Experiments}
\subsection{Computational Resources and Implementations}
All experiments were conducted on a server equipped with an Intel(R) Xeon(R) CPU E5-2698 v4 @ 2.20GHz CPU, 512GB RAM, and a Tesla V100 GPU. For models such as the Exponomial, Markov Chain, and Stochastic Preference, the solutions to the optimization problems were determined using the minimize method from the {\it SciPy} library, and the implementation follows that of \citet{berbeglia2022comparative}. Conversely, models representable as neural networks were trained utilizing the {\it PyTorch} library. For detailed implementation, refer to the anonymized repository: https://anonymous.4open.science/r/choicemodel

\subsection{Detailed Results for IRI Academic Datasets}

The experiment results are summarized in  \cref{tab:4weeks_table} and \cref{tab:52weeks_table}. A detailed table, which presents the cross-entropy loss on the test data for each choice model across all 31 product categories using short-term(4-week) and long-term (52-week) data, is provided in \cref{tab:4weeks_full_result} and \cref{tab:52weeks_full_result}.

\input{sections/table_all_products_4weeks}
\input{sections/table_all_products_52_weeks}

%% file: sections/table_all_products_4weeks.tex
\begin{table}[h]
\caption{Comparative Analysis of Cross-Entropy Loss on Short-Term (4-Week) Data Across Various Models for 31 Product Categories. The models include Multinomial Logit (MNL), Mixed Multinomial Logit (M-MNL), Halo MNL, Exponential (Exp), Monte Carlo (MC), Stochastic Preference (SP), and Low-Rank Halo MNL.}
\label{tab:4weeks_full_result}
\vskip 0.15in
\centering
\begin{tabular}{cccccccc}
\toprule
Product category & \textbf{MNL} & \textbf{M-MNL} & \textbf{Halo MNL} & \textbf{Exp} & \textbf{MC} & \textbf{SP} & \textbf{Low-Rank Halo MNL}\\
\midrule
\textit{Beer}      & 2.2637                     & 2.2726                            & 2.2375                & 2.3227             & 2.2989             & 2.307             & 2.2363                      \\
\textit{Blades}    & 1.3648                     & 1.4855                            & 1.3541                & 1.3131             & 1.3903             & 1.3126            & 1.3142                      \\
\textit{Carb. beverages}   & 1.8992                     & 1.9193                            & 1.93                  & 1.8574             & 1.9115             & 1.8694            & 1.9156                      \\
\textit{Cigarette}    & 1.4944                     & 1.5473                            & 1.4665                & 1.5241             & 1.5281             & 1.4888            & 1.435                       \\
\textit{Coffee}    & 2.3982                     & 2.4333                            & 2.3838                & 2.3727             & 2.3724             & 2.3755            & 2.3763                      \\
\textit{Cold cereal}   & 2.117                      & 2.1283                            & 2.0966                & 2.1048             & 2.1234             & 2.1177            & 2.068                       \\
\textit{Deodorant}      & 2.286                      & 2.3206                            & 2.2682                & 2.3382             & 2.2907             & 2.275             & 2.2517                      \\
\textit{Diapers}   & 1.1186                     & 1.2376                            & 1.106                 & 1.0932             & 1.1045             & 1.0938            & 1.0922                      \\
\textit{Facial Tissue}   & 1.3868                     & 1.61                              & 1.3072                & 1.2831             & 1.2499             & 1.3024            & 1.2499                      \\
\textit{Frozen Dinners}  & 2.5764                     & 2.5814                            & 2.6036                & 2.5834             & 2.5634             & 2.5696            & 2.6141                      \\
\textit{Frozen Pizza}   & 2.1192                     & 2.1596                            & 2.1191                & 2.0495             & 2.0749             & 2.1591            & 2.1212                      \\
\textit{House cleaner}   & 2.6119                     & 2.6184                            & 2.604                 & 2.7018             & 2.5778             & 2.5815            & 2.5899                      \\
\textit{Hotdog}    & 2.2147                     & 2.3025                            & 2.1454                & 2.0366             & 2.0292             & 2.1461            & 2.0588                      \\
\textit{Detergent}   & 1.5667                     & 1.6362                            & 1.5068                & 1.5151             & 1.5079             & 1.5877            & 1.485                       \\
\textit{Margarine}  & 2.2114                     & 2.3039                            & 2.1882                & 2.1361             & 2.188              & 2.1632            & 2.1586                      \\
\textit{Mayonnaise}      & 1.4549                     & 1.6263                            & 1.4303                & 1.4557             & 1.3994             & 1.5908            & 1.3463                      \\
\textit{Milk}      & 2.2897                     & 2.3796                            & 2.2105                & 2.1956             & 2.1583             & 2.1648            & 2.2389                      \\
\textit{Mustard/ketchup}  & 2.4727                     & 2.5065                            & 2.4408                & 2.4198             & 2.4239             & 2.451             & 2.4225                      \\
\textit{Paper towels}   & 1.7824                     & 1.9528                            & 1.766                 & 1.6734             & 1.6664             & 1.6853            & 1.6779                      \\
\textit{Peanut butter}  & 1.5367                     & 1.7446                            & 1.4966                & 1.4311             & 1.4691             & 1.5181            & 1.4164                      \\
\textit{Photography}     & 1.4929                     & 1.7722                            & 1.4316                & 1.3985             & 1.4094             & 1.4316            & 1.3992                      \\
\textit{Razors}    & 1.2552                     & 1.3447                            & 1.2361                & 1.001              & 1.045              & 1.0893            & 1.0917                      \\
\textit{Salty snacks}  & 2.231                      & 2.2485                            & 2.2072                & 2.1959             & 2.2054             & 2.1902            & 2.2186                      \\
\textit{Shampoo}     & 2.3964                     & 2.4317                            & 2.3542                & 2.3607             & 2.3794             & 2.3669            & 2.3423                      \\
\textit{Soup}      & 2.0374                     & 2.0528                            & 2.014                 & 2.0045             & 2.0262             & 2.0179            & 2.0109                      \\
\textit{Spaghetti}  & 2.1919                     & 2.2608                            & 2.1563                & 2.1092             & 2.1485             & 2.1542            & 2.1322                      \\
\textit{Sugar sub.}  & 1.6618                     & 1.8188                            & 1.6657                & 1.5513             & 1.6384             & 1.5977            & 1.5513                      \\
\textit{Toilet tissue}   & 1.8939                     & 2.0385                            & 1.8845                & 1.8314             & 1.8511             & 1.8397            & 1.8552                      \\
\textit{Toothbrush}   & 2.596                      & 2.6251                            & 2.5615                & 2.5561             & 2.5734             & 2.5748            & 2.5518                      \\
\textit{Toothpaste}   & 1.8519                     & 1.8913                            & 1.8361                & 1.7388             & 1.8103             & 1.7993            & 1.8216                      \\
\textit{Yogurt}    & 1.7372                     & 1.7698                            & 1.7344                & 1.6234             & 1.6083             & 1.6849            & 1.6834\\            
\bottomrule
\end{tabular}
\vspace{-2mm}
\end{table}

%% file: sections/table_all_products_52_weeks.tex
\begin{table}[h]
\centering
\caption{Comparative Analysis of Cross-Entropy Loss on Long-Term (52-Week) Data Across Various Models for 31 Product Categories. The models include Multinomial Logit (MNL), Mixed Multinomial Logit (M-MNL), Halo MNL, Exponential (Exp), Monte Carlo (MC), Stochastic Preference (SP), and Low-Rank Halo MNL}
\label{tab:52weeks_full_result}
\vskip 0.15in
\centering
\begin{tabular}{cccccccc}
\toprule
Product category & \textbf{MNL} & \textbf{M-MNL} & \textbf{Halo MNL} & \textbf{Exp} & \textbf{MC} & \textbf{SP} & \textbf{Low-Rank Halo MNL}\\
\midrule
\textit{Beer}            & 2.3421       & 2.347              & 2.3183            & 2.3395              & 2.3569                & 2.3439                         & 2.3166                     \\
\textit{Blades}          & 1.3323       & 1.3589             & 1.3239            & 1.3201              & 1.3997                & 1.3205                         & 1.3198                     \\
\textit{Carb. beverages} & 1.8873       & 1.9029             & 1.8747            & 1.8874              & 1.9633                & 1.9632                         & 1.87                       \\
\textit{Cigarette}       & 1.5028       & 1.5251             & 1.498             & 1.4802              & 1.5301                & 1.5834                         & 1.4751                     \\
\textit{Coffee}          & 2.389        & 2.4116             & 2.352             & 2.3832              & 2.394                 & 2.3893                         & 2.3463                     \\
\textit{Cold cereal}     & 2.2158       & 2.2251             & 2.2062            & 2.3564              & 2.2656                & 2.2346                         & 2.2038                     \\
\textit{Deodorant}       & 2.3277       & 2.3412             & 2.3201            & 2.3495              & 2.3639                & 2.3381                         & 2.3234                     \\
\textit{Diapers}         & 1.1036       & 1.1112             & 1.0942            & 1.0809              & 1.1236                & 1.0818                         & 1.0905                     \\
\textit{Facial Tissue}   & 1.2999       & 1.3389             & 1.2857            & 1.2793              & 1.3268                & 1.2812                         & 1.285                      \\
\textit{Frozen Dinners}  & 2.5215       & 2.5276             & 2.5182            & 2.6956              & 2.5282                & 2.5399                         & 2.5152                     \\
\textit{Frozen Pizza}    & 2.0889       & 2.1167             & 2.0283            & 2.0555              & 2.0789                & 2.1367                         & 2.033                      \\
\textit{House cleaner}   & 2.5497       & 2.5561             & 2.5414            & 2.5393              & 2.5487                & 2.5734                         & 2.5486                     \\
\textit{Hotdog}          & 2.1869       & 2.2428             & 2.1553            & 2.168               & 2.1851                & 2.2031                         & 2.1533                     \\
\textit{Detergent}       & 1.6597       & 1.6815             & 1.6488            & 1.6494              & 1.74                  & 1.8299                         & 1.6457                     \\
\textit{Margarine}       & 2.1486       & 2.1759             & 2.1302            & 2.1229              & 2.1802                & 2.1759                         & 2.1187                     \\
\textit{Mayonnaise}      & 1.4934       & 1.5432             & 1.4703            & 1.4629              & 1.5798                & 1.6527                         & 1.4568                     \\
\textit{Milk}            & 2.1412       & 2.1936             & 2.0685            & 2.14                & 2.138                 & 2.1979                         & 2.0924                     \\
\textit{Mustard/ketchup} & 2.4154       & 2.455              & 2.3896            & 2.3812              & 2.393                 & 2.4157                         & 2.3911                     \\
\textit{Paper towels}    & 1.8285       & 1.8627             & 1.8125            & 1.79                & 1.8253                & 1.7991                         & 1.7996                     \\
\textit{Peanut butter}   & 1.6188       & 1.6718             & 1.6045            & 1.5758              & 1.6834                & 1.7639                         & 1.5744                     \\
\textit{Photography}     & 1.4823       & 1.604              & 1.4665            & 1.43                & 1.4559                & 1.4818                         & 1.4323                     \\
\textit{Razors}          & 0.9406       & 1.0368             & 0.9297            & 0.8827              & 0.9348                & 0.907                          & 0.883                      \\
\textit{Salty snacks}    & 2.031        & 2.0434             & 2.0069            & 2.0342              & 2.066                 & 2.0712                         & 2.0068                     \\
\textit{Shampoo}         & 2.4662       & 2.4832             & 2.455             & 2.4601              & 2.482                 & 2.4722                         & 2.4563                     \\
\textit{Soup}            & 2.1366       & 2.144              & 2.1284            & 2.5279              & 2.15                  & 2.1512                         & 2.1251                     \\
\textit{Spaghetti}       & 2.2368       & 2.261              & 2.2195            & 2.2289              & 2.2656                & 2.2558                         & 2.2158                     \\
\textit{Sugar sub.}      & 1.5231       & 1.5759             & 1.4863            & 1.4721              & 1.5435                & 1.6321                         & 1.4693                     \\
\textit{Toilet tissue}   & 1.914        & 1.9421             & 1.8998            & 1.8876              & 1.9418                & 1.9052                         & 1.8832                     \\
\textit{Toothbrush}      & 2.5752       & 2.5909             & 2.5617            & 2.5691              & 2.6049                & 2.5873                         & 2.5615                     \\
\textit{Toothpaste}      & 1.8422       & 1.8592             & 1.8316            & 1.9929              & 1.9244                & 1.9572                         & 1.8367                     \\
\textit{Yogurt}          & 1.9069       & 1.9356             & 1.8703            & 3.1619              & 1.9414                & 1.9838                         & 1.8723\\
\bottomrule
\end{tabular}
\vspace{-2mm}
\end{table}